\documentclass[lettersize,journal]{IEEEtran}
\usepackage{amsmath,amsfonts}
\usepackage{algorithmic}
\usepackage{algorithm}
\usepackage{array}
\usepackage[caption=false,font=normalsize,labelfont=sf,textfont=sf]{subfig}
\usepackage{textcomp}
\usepackage{stfloats}
\usepackage{url}
\usepackage{verbatim}
\usepackage{graphicx}
\usepackage{booktabs}
\usepackage{cite}
\usepackage{graphicx}

\usepackage{xspace}
\usepackage{bbold}
\usepackage{amsfonts}
\newcommand{\lb}[1]{\underline{#1}}
\newcommand{\ub}[1]{\overline{#1}}
\newcommand{\ten}[0]{\emph{$\boldsymbol{\tau}$}\xspace}
\DeclareMathOperator*{\argmin}{argmin}
\usepackage{tikz}
\usetikzlibrary{arrows.meta, shapes.geometric, positioning}

\newcommand{\vect}[1]{\boldsymbol{\mathbf{#1}}}
\newcommand{\mat}[1]{\boldsymbol{\mathbf{#1}}}

\usepackage{amsthm}
\theoremstyle{plain}
\newtheorem{theorem}{Theorem}
\newtheorem{lemma}{Lemma}

\newtheoremstyle{myremark}
  {\topsep}      % Space above
  {\topsep}      % Space below
  {\itshape}     % Body font
  {}             % Indent
  {\bfseries}    % Head font
  {.}            % Punctuation after head
  { }            % Space after head
  {}             % Head spec

\theoremstyle{myremark}

\newtheorem{remark}{Remark}

\newtheorem{assumption}{Assumption}

\begin{document}

\title{X-ACTA: eXtended Analytic Center Tension distribution Algorithm for fixed and mobile cable-driven-parallel-robot}

\author{
    Domenico~Dona',
    Vincenzo~Di~Paola,
    Alberto~Trevisani
    and~Matteo~Zoppi
    
    \thanks{Domenico Dona' is with the Faculty of Engineering, Free University of Bozen-Bolzano, Via Bruno Buozzi 1, Bozen 39100, Italy (e-mail: domenico.dona@unibz.it).}
    \thanks{Vincenzo Di Paola and Matteo Zoppi are with Department of Mechanical, Energy, Management and Transportation Engineering (DIME), University of Genova, Via alla Opera Pia 15, Genova 16143, Italy (e-mail: vincenzo.dipaola@edu.unige.it; matteo.zoppi@unige.it).}
    \thanks{Alberto Trevisani is with the Department of Management and Engineering (DTG), University of Padua, Stradella San Nicola 3, Vicenza 36100, Italy (e-mail: alberto.trevisani@unipd.it).}
    \thanks{Corresponding author: Domenico Dona'}
}

\maketitle

\begin{abstract}
Steering Cable-Driven Parallel Robots (CDPRs) beyond their Wrench-Feasible Workspace (WFW) augments their capabilities in challenging scenarios such as during aggressive maneuvers or following a cable failure. In this context, although the determination of cable tensions is a well-studied topic, only a few approaches address these scenarios.
Therefore, this paper introduces an extended version of the Analytic Center method as a criterion for selecting cable tensions outside the WFW while maintaining differentiability and including non-linear constraints. Notably, the proposed method maintains continuous and differentiable tension profiles, ensures fast real-time convergence to a unique solution, and, in contrast to other slack-based formulations, relegates wrench errors to a negligible area of the WFW. Its superiority in terms of smoothness and wrench error is confirmed via Pareto dominance with respect to the leading state-of-the-art method. Lastly, the effectiveness of the method is demonstrated through numerical experiments.
\end{abstract}


\date{}

\section{Introduction}
Cable-Driven Parallel Robots (CDPRs) are typically composed of a moving platform, namely a rigid body, actuated by cables exerting forces on it. The selection of cable tensions in CDPRs has been a core research topic since their birth~\cite{pott2018cable}.

In particular, overconstrained CDPRs, where the number of cables exceeds the payload's Degrees of Freedom ($\mathrm{DoF}$s), require the selection of a tension set among infinitely many feasible ones. This choice is typically made by solving an optimization problem known as the Tension Distribution Algorithm (TDA). The existing approaches differ in the choice of optimization objective, constraints, and solution method, typically tailored to the task requirements \cite{gouttefarde2015versatile}.

Generally, $p$-norm-based metrics have been widely adopted as objectives for TDAs \cite{Gosselin2011}. The 1-norm, for example, is key to several studies \cite{Shiang2000,Oh2003,Borgstrom2009}. Such problems are typically solved using Linear Programming (LP) approaches due to fast convergence; however, the jump between the vertices of the feasible polyhedron, resulting from the $1-$norm minimization, often results in discontinuous tension profiles. Similarly, discontinuities can occur when using the $\infty$-norm \cite{Gosselin2011}. Because continuity is a must to ensure precision while executing tasks, the criteria have been reformulated employing $p$-norms where $1<p<\infty$ \cite{Verhoeven2002}. 

\begin{figure}[t]
    \centering
    \includegraphics[width=\linewidth]{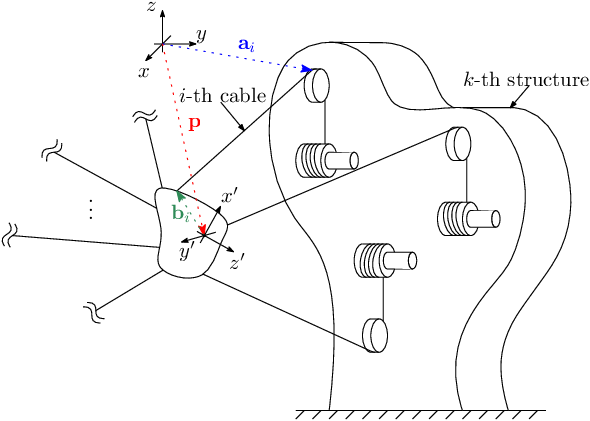}
    \caption{Scheme of a generic cable-driven parallel robot.}
    \label{fig:cdprscheme}
\end{figure}

Consequently, a widely adopted objective is the 2-norm of the tension vector \cite{Agahi2009,Pott2009,Taghirad2011} or its distance with respect to a reference tension, usually taken as the mean between the minimum and maximum ones \cite{Gosselin2011,Pott2013,Cote2016}. The popularity of these methods relies on the possibility of efficiently solving the associated optimization problem, using either closed-form solutions \cite{Pott2009,Pott2013} or specialized quadratic programming (QP) solvers \cite{Cote2016}. The former approaches are traditionally valued for their computational efficiency, albeit at the expense of limited workspace coverage. However, this advantage has become less significant as modern iterative methods have achieved real-time performance while maintaining full workspace coverage \cite{Cote2016,Ueland2020,DiPaola2025}.

Beyond $p-$norm-based methods, other relevant approaches include the Dijkstra alternating-projection algorithm proposed in \cite{Hassan2008,Hassan2011}. Additionally, in~\cite{Mikelsons2008}, the barycenter of the feasible tension polyhedron is selected as the optimal solution, thereby guaranteeing continuity.

However, the aforementioned methods are unable to produce differentiable tension profiles, and some are further constrained by the limited number of cables that can be employed. Remarkably, the Analytic Center (AC) method~\cite{DiPaola2024} and the method proposed in \cite{Ueland2020} (hereafter referred to as the NTNU method) exhibit this feature by incorporating tension bound constraints into the objective function via barrier functions, e.g., logarithmic functions. 

Interestingly, the latter also operates beyond the Wrench-Feasible Workspace (WFW) by introducing slack variables, at the cost of inducing a wrench error even within the WFW. Although this capability is shared with other existing methods \cite{muller2014analysis, Boumann2020, Boumann2020_2}, these rely on heuristic approaches, thus avoiding wrench errors inside the WFW at the cost of suboptimality beyond it. Moreover, none of them provides differentiable tension profiles when transitioning across WFW boundaries, despite this being crucial for haptic tasks~\cite{Cote2016} or cable breakage management~\cite{DiPaola2024_2}.

To the best of the authors' knowledge, no method in the literature simultaneously (i) ensures differentiable tension profiles, (ii) operates beyond the WFW avoiding wrench errors inside it, and lastly, (iii) allows for the inclusion of non-linear constraints.

In this paper, a TDA that meets all the aforementioned requirements is proposed. The method relies on a two-stage formulation: first, a relaxed version of the AC method is solved, and secondly, when the solution falls outside the (relaxed) tension bounds, a subsequent slacked optimization problem is solved. Notice that the switching does not affect the TDA solutions. Indeed, the tension profiles are proven to remain differentiable while switching, even in the presence of non-linear constraints.

The proposed TDA is validated through simulated experiments against the NTNU method \cite{Ueland2020}, which represents the existing state-of-the-art for ensuring differentiable tension profiles while operating beyond the WFW. The results demonstrate reduced convergence times and Pareto dominance regarding the trade-off between high-frequency content and wrench tracking error.

The remainder of this paper is organized as follows. Section~\ref{sec:modelling} details the kinematic and dynamic modeling equations for CDPRs. Subsequently, a motivating example comparing state-of-the-art approaches is discussed in Section~\ref{sec:motexample}, while the proposed method is presented in Section~\ref{sec:method}. Section~\ref{sec:simul} presents simulation results that demonstrate the effectiveness of the proposed method, while Section~\ref{sec:conclusion} closes the paper and highlights future works.

\section{CDPR Modelling}\label{sec:modelling}
When working with CDPRs, the platform dynamics is typically described as~\cite{pott2018cable}
\begin{equation}\label{eq:dyn}
    \mat{M}(\vect{q}) \ddot{\vect{q}} + \mat{C}(\vect{q}, \dot{\vect{q}}) \dot{\vect{q}} + \vect{g}(\vect{q}) = \mat{W}(\vect{q}) \ten + \vect{w}_{\mathrm{ext}},
\end{equation}
where $\vect{q} = [\vect{p}^\top, \; \vect{\psi}^\top]^\top$ denotes the generalized coordinate vector, which includes both the position $\vect{p}$ and the orientation $\vect{\psi}$ of the end-effector. The matrices $\mat{M}$ and $\mat{C}$ represent the mass and Coriolis terms, respectively; $\vect{g}$ accounts for the gravitational effects, while $\vect{w}_\mathrm{ext}$ stands for the external wrench. The vector $\ten \in \mathbb{R}^m$ collects the tensions of the $m$ cables, and $\mat{W}$ is the so-called wrench matrix, defined as
\begin{equation}
\mat{W} = 
\begin{bmatrix}
\vect{u}_1 & \dots & \vect{u}_m \\
\vect{b}_1 \times \vect{u}_1 & \dots & \vect{b}_m \times \vect{u}_m
\end{bmatrix},
\end{equation}
where $\vect{u}_i \in \mathbb{R}^d$ is the unit vector along the $i$-th cable and $\vect{b}_i \in \mathbb{R}^d$ is the corresponding attachment point on the platform\footnote{$d$ is the dimension of the physical space of the robot, e.g. $d=2$ for a planar robot and $d=3$ for a spatial one. With abuse of notation, the cross product of 2D vectors is intended as the third component of the 3D product.}.  
The wrench required for the equilibrium of the platform is then 
\begin{equation}
    \vect{w}_e = \mat{M}(\vect{q}) \ddot{\vect{q}} + \mat{C}(\vect{q}, \dot{\vect{q}}) \dot{\vect{q}} + \vect{g}(\vect{q}) - \vect{w}_{\mathrm{ext}} \in \mathbb{R}^n,
\end{equation}
where $n$ is the number of $\mathrm{DoF}$ of the platform. Thus, at a fixed time, for a given kinematic triplet $\{\vect{q}, \dot{\vect{q}}, \ddot{\vect{q}}\}$, the inverse dynamics problem in \eqref{eq:dyn} reduces to
\begin{equation}\label{eq:dynequi}
    \Sigma = \{ \ten \mid \mat{W}\ten = \vect{w}_e \},
\end{equation}
where the dependence on $\{\vect{q}, \dot{\vect{q}}, \ddot{\vect{q}}\}$ is omitted for brevity.  
Clearly, if the Degree of Redundancy ($\mathrm{DoR} = m - n$) is strictly positive, $\Sigma$ forms an affine subspace.

Furthermore, to ensure cable tautness while avoiding breakage, the tensions $\ten$ must be bounded above and below. Therefore, an $m$-dimensional convex hypercube $\Pi$ defining the domain of \emph{feasible} tensions is introduced as
\begin{equation}
\Pi = \big\{ \ten \mid \vect{0} \prec \lb\ten \preceq \ten \preceq \ub\ten \big\},
\label{tenslim}
\end{equation}
where $\lb\ten, \ub\ten \in \mathbb{R}^{m,+}$ denote the lower and upper tension limits, respectively. Without loss of generality, these limits are assumed to be identical across all cables. Consequently, the set of feasible solutions $\Gamma$ satisfying both Eq.~\eqref{eq:dyn} and Eq.~\eqref{tenslim} is given by
\begin{equation}
\Gamma = \Sigma \cap \Pi.
\end{equation}
In what follows, the condition $\Gamma \neq \emptyset$ represents the nominal condition for a TDA to operate, and the focus of this work is to extend the TDA's capability to generate approximate solutions when $\Gamma = \emptyset$.

\begin{remark}
 Notice that an approximate solution implies that the desired wrench can only be partially generated within the constraints of $\Pi$. Thus, the resulting discrepancy is here defined  as the wrench error or constraints violation introduced by the TDA.
\end{remark}

\section{Motivating Example}\label{sec:motexample}
\subsection{State-of-the-art methods}
This section reviews existing state-of-the-art methods for computing tension solutions even outside the WFW. We do not consider LP-based methods as they do not ensure continuity of the solution \cite{Gosselin2011}.

\noindent

\begin{figure*}[t]
\centering
\begin{minipage}[t]{0.31\textwidth}
    \centering
    \includegraphics[width=\linewidth]{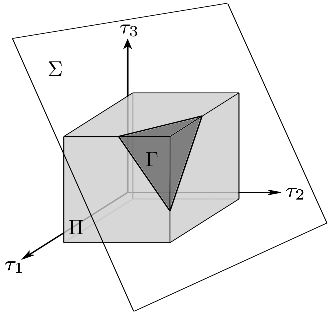}
    \vspace{-2mm} % Opzionale: stringe lo spazio prima della lettera
    \centerline{(a)}
    \label{feas}
\end{minipage}\hfill
\begin{minipage}[t]{0.31\textwidth}
    \centering
    \includegraphics[width=\linewidth]{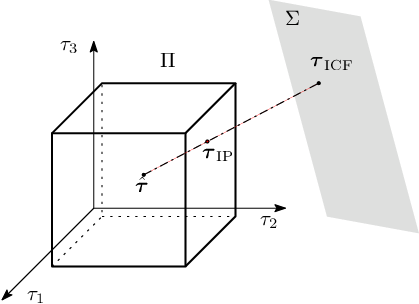}
    \vspace{-2mm}
    \centerline{(b)}
    \label{puncture}
\end{minipage}\hfill
\begin{minipage}[t]{0.31\textwidth}
    \centering
    \includegraphics[width=\linewidth]{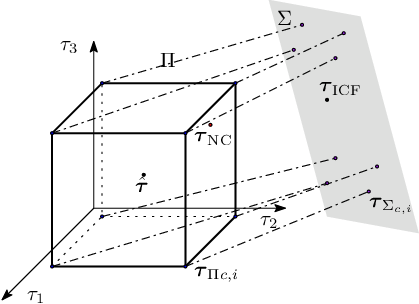}
    \vspace{-2mm}
    \centerline{(c)}
    \label{discsol}
\end{minipage}
\caption{(a) Geometric representation of the feasible solution space defined by $\Pi, \Sigma$, and $\Gamma$ ($m=3, n=2$). (b) Conceptual illustration of the Improved Puncture (IP) method. (c) Conceptual illustration of the Nearest Corner (NC) method.}
\label{fig:cdpr-3in1}
\end{figure*}

Hence, one class of the considered methods relies on the Improved Closed Form method (ICF) \cite{Pott2013}. Here, the tension components lying outside $\Pi$  are set to the nearest tension bound. However, the ICF can manage a maximum of DoR components exceeding the tension limits before failing. 
Two stratagems have been proposed to overcome these limitations. The first is the so-called Improved Puncture (IP) method~\cite{muller2014analysis}, while the other is the Nearest Corner (NC) method~\cite{Boumann2020}.

\indent
The former constructs a ray between the ICF solution and the center of $\Pi$, denoted as $\hat{\ten}$. The algorithm then selects the intersection point between this ray and the boundary of $\Pi$, as depicted in Fig.~\ref{puncture}. Although this method ensures rapid computation, the resulting tension profiles are non-differentiable, and the solution lacks guaranteed optimality regarding minimal wrench error. 

\indent
On the other hand, the NC method is based on geometric considerations. The algorithm projects the vertices of the hypercube $\Pi$ onto the affine space $\Sigma$. Then, it computes a weighted sum of these projections—where the weights are inversely proportional to the projection distances—to determine the final tension vector (see Fig.~\ref{discsol}). Mathematically, the method computes
\begin{equation}
d_i = || \ten_{\Pi_{c,i}} - \ten_{\Sigma_{c,i}} ||_2, \quad g_i = \left( \frac{L}{d_i} \right)^p,
\end{equation}
where $d_i$ is the Euclidean distance between the $i$-th vertex of the tension box, $\ten_{\Pi_{c,i}}$, and its projection onto the affine space, $\ten_{\Sigma_{c,i}}$. The weighting factors $g_i$ are calculated as the $p$-th power of the ratio between the sum of all distances $L = \sum^{2m}_{i=1} d_i$ and the individual distance $d_i$. Therefore, the tension vector is given by
\begin{equation}
    \ten_{\mathrm{NC}} = \sum^{2m}_{i=1} \left( \ten_{\Sigma_{c,i}} \frac{g_i}{G}\right),
\end{equation}
where $G=\sum^{2m}_1 g_i$. Note that if $\Sigma$ intersects $\Pi$ exactly at a vertex $\ten_{\Pi_{c,i}}$, the corresponding distance $d_i$ becomes zero. This causes $g_i \to \infty$, rendering the mathematical formulation singular and resulting in an unbounded tension calculation. Furthermore, the resulting tension profiles can exhibit discontinuities when $\Sigma$ is parallel to the faces of $\Pi$, as discussed in \cite{DiPaola2024_2}. 

Another relevant approach, here denoted as the Laval method~\cite{Cote2016}, introduces a slack variable to relax the equality constraint in Eq.~\eqref{eq:dynequi}, thereby extending it to work when $\Gamma = \emptyset$. The QP formulation proposed there reads as 
\begin{equation}
\begin{aligned}
  \ten_{\mathrm{Laval}} = \argmin \quad & 
        \vect{s}_1^T \mat{D}_1 \vect{s}_1 
        + (\ten - \hat{\ten})^T \mat{D}_2 (\ten - \hat{\ten}) \\
    \text{s.t.} \quad & 
        \mat{W}\ten + \vect{s}_1 = \vect{w}_e, \\
    & \ten \in \Pi,
\end{aligned}
\end{equation}
\noindent
where $\vect{s}_1$ is the slack variable, and $\mat{D}_1, \mat{D}_2$ are weighting matrices (typically diagonal). Note that the elements of $\mat{D}_1$ must be heavily penalized relative to $\mat{D}_2$ to minimize wrench violations inside the WFW. However, this method inherently introduces (small) wrench errors even under nominal conditions $\Gamma \neq \emptyset$ and yields non-differentiable tension profiles.

\indent
On the contrary, the NTNU method \cite{Ueland2020} in his slacked formulation ensures differentiable tension profiles even beyond the WFW. Specifically, this is obtained by combining a barrier-function with a slack formulation by splitting the objective function into two terms
\begin{equation}
    \begin{aligned}
        \ten_{\mathrm{NTNU}} = \argmin \quad & 
             g(\ten) + g_s(\vect{s}) \\
         \textrm{s.t.} \quad &  \mat{W} \ten = \vect{w}_e + \vect{s},
    \end{aligned}
\end{equation}
with
\begin{subequations}
\begin{align}&
\begin{aligned}
    g(\ten) = &c_0 \lVert \ten - \hat{\ten} \rVert_p^p + \\
    &- \sum_{i=1}^m \bigl(c_{1,i} \log(\ten_i - \lb \ten) 
    + c_{2,i} \log(\ub \ten - \ten_i )\bigr)
\end{aligned} \\
&g_s(\vect{s}) =
\sum_{j=1}^d \left( b_j \sqrt{s_j^2 + \epsilon_j} + s_j^2 \right)
\end{align}
\end{subequations}
where $c_0 = 2^p/(\ub \tau - \lb \tau)^p$, $c_{1,i}$, and $c_{2,i}$ are weighting parameters, and $b_j$ and $\epsilon_j$ are tuning parameters governing the error and smoothness of the solution. Practically, $g(\ten)$ dominates the optimization under nominal conditions, while $g_s(\vect{s})$ takes weight when $\Gamma = \emptyset$ has non-zero influence even when $\Gamma \neq \emptyset$. Furthermore, $g_s(\mathbf{s})$ is intentionally constructed as an approximate 1-norm, ensuring differentiable tension profiles. 

\subsection{Numerical example}

\begin{figure*}[t!]
\centering
\includegraphics[width=\linewidth]{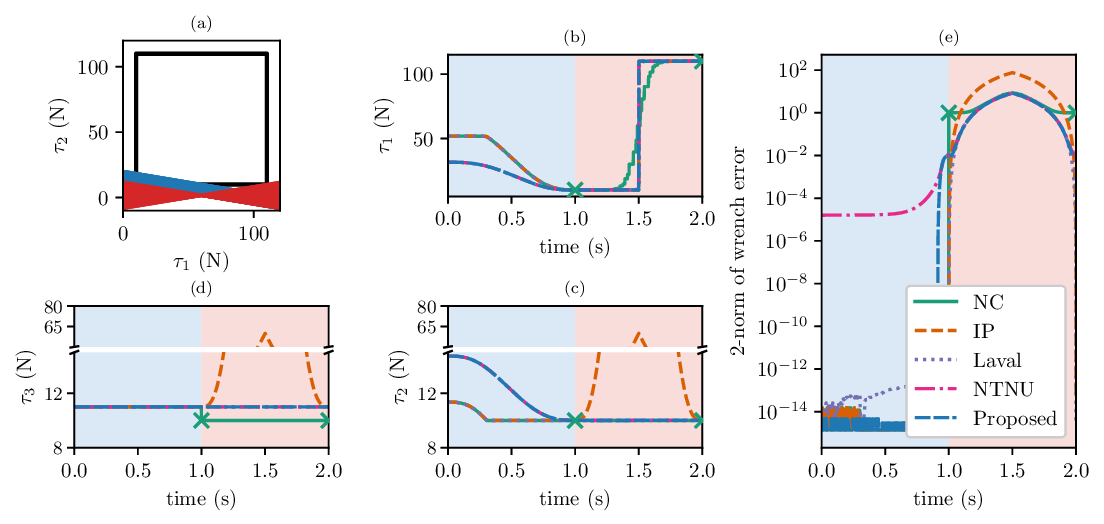}
\caption{A comparison of the Nearest Corner, Puncture, Laval and NTNU methods. The cross marks indicate where the NC method fails.}
\label{fig:comparisonexample}
\end{figure*}

In this section, a specific benchmark trajectory is considered to illustrate the strengths and weaknesses of the aforementioned methods, with a preview of the performance of the proposed TDA.

A simple $2-\mathrm{DoF}$, $3$-cable robot with tension limits of $[10, 110]$~N serves as the testbed.
The test trajectory comprises two phases. In the first phase, shown in blue in Fig.~\ref{fig:comparisonexample}a, the wrench matrix 
\begin{equation}
    \mat{W}_0=\begin{bmatrix}
        -1/6  & 1 & 0 \\
        0 & 0 & 1
    \end{bmatrix},
\end{equation}
is kept constant while the required wrench is increased until the upper-left boundary of the feasible region is reached. Specifically, the required wrench is defined as $\vect{w}_e =\vect{w}_{e,0} + \vect{w}(t)$, where $\vect{w}_{e,0} = [110, 11]^\top$~N, and $\vect{w}(t) = [w(t), 0]^\top$, with $w(t)$ being a quintic polynomial spanning from $0$ to $25/3$~N and having zero velocity and acceleration at the boundaries. The second phase (shown in red in Fig.~\ref{fig:comparisonexample}a) rotates the affine space $\Sigma$ around the midpoint of the first cable's tension range, $\ten_1 = (10 +110)/2$, until it reaches the opposite slope. This has been built on purpose, so that the entire first phase consists of feasible points (thus $\Gamma \neq \emptyset$), while for the second phase, only the initial and final points are feasible. 

The resulting tension profiles for the evaluated methods are depicted in Fig.~\ref{fig:comparisonexample}b. As expected, only the NTNU ensures differentiable tension profiles. On the other hand, the Nearest Corner Method exhibits two critical flaws; one can be seen where tensions are discontinuous at $t=1.5$~s, while the second is revealed at the domain corners, where $d_i = 0$ and the solver fails due to the singularity. Both of them are indicated by cross symbols in Fig.~\ref{fig:comparisonexample}b. 

Another comparison metric is given by the wrench violation inside and outside the WFW. Figure~\ref{fig:comparisonexample}c reports that methods relying on a slack variable formulation (e.g., Laval and NTNU) inherently introduce a wrench error even when operating safely inside the WFW. Eventually, although the NTNU produces differentiable tensions, it produces non zero errors even in nominal conditions. 

\section{Proposed method}\label{sec:method}

Section~\ref{sec:motexample} highlighted that existing methods either rely on heuristic projections onto the WFW or employ slack-based formulations. The former lacks optimality guaranties, while the latter introduces non-zero error even within the WFW. Therefore, this section presents a method that overcomes both limitations while ensuring differentiable tension profiles and enabling the inclusion of nonlinear inequality constraints.\\

To formalize the problem and reach the intended aim, a practical yet fundamental assumption regarding the imposed trajectory and its consequences is formulated as follows.

\begin{assumption}\label{A:1}
    Given a $k_1$-differentiable trajectory, the maps $t \mapsto \vect{w}_e(t)$ and $t \mapsto \mat{W}(t)$ are assumed to be $k_1$ times differentiable, with $\mat{W}$ full rank.
\end{assumption}

Below we recall the Analytic Center (AC) method~\cite{DiPaola2024}
\begin{equation}\label{eq:AC} 
\begin{aligned}
    \ten_{\mathrm{AC}} = 
    \argmin_{\ten} \quad &
        f(\ten) =
        \sum_{i=1}^m - \lb c_i \log(\tau_i - \lb \tau_i) - \ub c_i \log(\ub \tau_i - \tau_i) \\
    \text{s.t.} \quad & \mat{W}\,\ten = \vect{w}_e .
\end{aligned}
\end{equation}
where $\tau_i$ is the $i-$th component of $\ten$, while $\lb c_i$ and $\ub c_i$ are weighting parameters that can either be fixed or tuned during motion \cite{dona2025stiffness}.

\begin{remark}
    Problem~\eqref{eq:AC} admits a solution if and only if $\Sigma~\cap~\Pi \neq \emptyset$; thus, no solution exists outside the WFW.
\end{remark}

Now, to extend the AC method, a two-step approach is proposed. First, we relax the AC formulation to allow violation of $\Pi$ and subsequently, we include the slack variable only when needed.

Let $\delta>0$ be small and define the \textit{relaxed} tension box
\begin{equation}
    \Pi_\delta = \big\{ \ten \mid \vect{0} \prec \lb\ten + \delta \vect{1} \preceq \ten \preceq \ub\ten - \delta \vect{1} \big\}.
\end{equation}

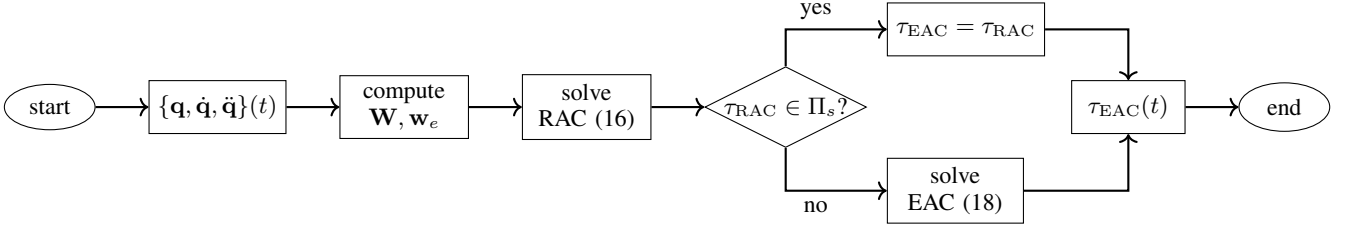
\begin{figure*}[t]
\centering
\begin{tikzpicture}[
    node distance=0.4cm and 0.7cm, % Aumentata distanza orizzontale di 2mm per slanciare
    every node/.style={font=\small}, 
    startstop/.style={ellipse, draw, minimum width=1.2cm, minimum height=0.6cm},
    process/.style={rectangle, draw, minimum width=1.8cm, minimum height=0.7cm, align=center},
    decision/.style={diamond, draw, aspect=2, inner sep=0pt, minimum width=1.8cm},
    arrow/.style={->, thick}
]

% Nodes
\node (start) [startstop] {start};
\node (state) [process, right=of start, minimum width=1.7cm] {$\{\mathbf{q}, \dot{\mathbf{q}}, \ddot{\mathbf{q}}\}(t)$};
\node (compute) [process, right=of state, minimum width=1.7cm] {compute \\ $\mathbf{W}, \mathbf{w}_e$};
\node (rac) [process, right=of compute, minimum width=1.7cm] {solve \\ RAC (16)};
\node (decision) [decision, right=of rac] {$\tau_{\mathrm{RAC}} \in \Pi_s$?};

% Branch "Yes" e "No" - offset verticale aumentato per evidenziare l'angolo a 90°
\node (eac_assign) [process, above right=0.4cm and 0.8cm of decision] {$\tau_{\mathrm{EAC}} = \tau_{\mathrm{RAC}}$};
\node (eac_solve) [process, below right=0.4cm and 0.8cm of decision] {solve \\ EAC (18)};

% Tau e End
\node (tau) [process, right=2.7cm of decision, minimum width=1.5cm] {$\tau_{\mathrm{EAC}}(t)$};
\node (end) [startstop, right=of tau] {end};

% Arrows
\draw [arrow] (start) -- (state);
\draw [arrow] (state) -- (compute);
\draw [arrow] (compute) -- (rac);
\draw [arrow] (rac) -- (decision);

% Decision Arrows - yes/no posizionati 
\draw [arrow] (decision.north) |- node[above, xshift=0.4cm] {yes} (eac_assign.west);
\draw [arrow] (decision.south) |- node[below, xshift=0.4cm] {no} (eac_solve.west);

% Merge Arrows
\draw [arrow] (eac_assign.east) -| (tau.north);
\draw [arrow] (eac_solve.east) -| (tau.south);
\draw [arrow] (tau) -- (end);

\end{tikzpicture}
\caption{Flow chart of the proposed method.}
\label{fig:scheme}
\end{figure*}

Then, a smooth relaxation of the logarithmic barrier is considered 
\begin{equation}\label{eq:relax}
    \beta_\delta(\Delta \tau) =
    \begin{cases}
        -\log(\Delta \tau), & \Delta \tau \ge \delta , \\
        -\log(\delta) + \displaystyle\sum_{j=1}^{k_2+1} \frac{(-1)^j}{j\,\delta^j}(\Delta \tau - \delta)^j ,
        & \Delta \tau < \delta .
    \end{cases}
\end{equation}
where $\delta$ specifies the distances of $\tau$ from the bounds and $k_2+1$ is the order of the Taylor expansion. 

By definition, $\beta_\delta(\Delta \tau)$ is $(k_2+1)$-times differentiable, and thus casting it into the AC objective returns the Relaxed Analytic Center (RAC) problem 
\begin{equation}\label{eq:def} 
\begin{aligned}
   \ten_{\mathrm{RAC}} =
    \argmin_{\ten} \quad &
        \tilde{f}(\ten) = 
        \sum_{i=1}^m \beta_\delta(\tau_i - \underline{\tau}_i) +
                         \beta_\delta(\overline{\tau}_i - \tau_i)\\
    \text{s.t.} \quad & \mat{W}\,\ten = \vect{w}_e .
\end{aligned}
\end{equation}

\begin{lemma}\label{lemma:1}
    The problem~\eqref{eq:def} admits a unique solution as the objective is strictly convex, and the constraints are affine. Moreover, the solution $\ten_{\mathrm{RAC}}(t)$ is $\min\{k_1, k_2\}$-times differentiable. 
\end{lemma}

\begin{proof}
The proof follows the same steps as Thm.~1 in~\cite{DiPaola2024}.
\end{proof}

\begin{remark}
Essentially, the introduction of $\Pi_\delta \subset \Pi$ and $\ten_{\mathrm{RAC}}$ avoids the wrench errors within the WFW, while providing a measure of proximity to the boundaries, indicating when the slack variable must be included in the formulation. Additionally, since $\delta$ is small, the subset of tension solutions that lie outside the relaxed box $\in \Pi \setminus \Pi_\delta$ is negligible and has no consequence for practical purposes. 
\end{remark}

Now, the second step necessary to extend the AC consists of accepting the violation of the desired wrench $\vect{w}_e$ for any solution outside $\Pi_\delta$. This can be done through the introduction of slack variables.

Hence, one can consider a $k_3$-differentiable function $\mu = \mu (\ten_{\mathrm{RAC}})$ that is zero if $\ten_{\mathrm{RAC}} \in \Pi_\delta$ and $>0$ otherwise , for example 

\begin{equation}\label{eq:mu}
\mu(\ten_{\mathrm{RAC}}) =
\begin{cases}
0 & \text{if} \, \ten_{\mathrm{RAC}} \in \Pi_\delta, \\
\tanh \!\left(
\sum_{i=1}^m 
\left|
\begin{aligned}
&\tau_{\mathrm{RAC},i} \\
&-\, c_{\Pi_\delta}(\tau_{\mathrm{RAC},i})
\end{aligned}
\right|^{k_3+1}
\right)
& \text{otherwise}.
\end{cases}
\end{equation}
where $c_{\Pi_\delta}$ clips the tension components on the boundary of $ \Pi_\delta$. 

\begin{lemma}\label{lemma:2}
    The function $\mu(t)$ is $\min\{k_1, k_2, k_3\}$-times differentiable.
\end{lemma}
\begin{proof}
    Note that $\mu$ is the composition of the mappings defined in~\eqref{eq:def} and~\eqref{eq:mu}. By Lemma~\ref{lemma:1}, these constituents possess regularity $C^{k_3}$ and $C^{\min\{k_1, k_2\}}$, respectively. The assertion then follows immediately from the Faà di Bruno formula.
\end{proof}

Finally, with this notion, one can define the Extended Analytic Center (EAC) as
\begin{equation}\label{eq:def2} 
\begin{aligned}
   \ten_{\mathrm{EAC}} =
    \argmin_{\ten, \vect{s}} \quad &
        f(\ten) + \eta \lVert \vect{s} \rVert_2^2 + \gamma \lVert \mu \vect{s} \rVert_2^2  \\
    \text{s.t.} \quad &
        \mat{W}\,\ten + \mu \vect{s} = \vect{w}_e ,
\end{aligned}
\end{equation}
where $\eta, \, \gamma>0$ are weighting coefficients and $\vect{s}\in\mathbb{R}^n$ are the slack variables.

To summarize the link between the problems defined so far, a schematic of the proposed algorithm is presented in Fig.~\ref{fig:scheme}. 

Notice that the computational effort required to solve problem~\eqref{eq:def} is not lost: when transitioning from a feasible to an unfeasible region along a trajectory, the solution $\ten_{\mathrm{RAC}}$ (with $\vect{s} = \vect{0}$) provides an effective warm start for problem~\eqref{eq:def2}.

Now, the main results of this section can be stated.

\begin{theorem}
Given \ref{A:1}, the problem~\eqref{eq:def2} admits a unique solution as the objective is strictly convex and the constraints are affine. Moreover, the solution $\ten_{\mathrm{EAC}} (t)$ is $\min\{k_1, k_2, k_3\}$-times differentiable.
\end{theorem}

\begin{proof}
Consider the Lagrangian of the problem~\eqref{eq:def2}
\begin{equation}
    \mathcal{L}(\ten, \vect{s}, \vect{\lambda}) = f(\ten) + \eta \lVert \vect{s} \rVert_2^2+ \gamma \lVert  \mu \vect{s} \rVert_2^2 + \vect{\lambda}^\top (\mat{W}\,\ten + \mu \vect{s} - \vect{w}_e),
\end{equation}
then the Karush-Kuhn-Tucker (KKT) conditions~\cite{Boyd2004convex} that return the optimal point for this problem are
\begin{equation}\label{eq:KKTcond}
\vect{r}(t, \vect{y})=
    \begin{cases}
        \nabla_{\ten} \mathcal{L} = \nabla_{\ten} f + \mat{W}^\top \vect{\lambda} = \vect{0} \\
        \nabla_{\vect{s}} \mathcal{L} = 2( \gamma \mu^2 + \eta) \vect{s} + \mu \vect{\lambda} = \vect{0} \\
        \nabla_{\vect{\lambda}} \mathcal{L} = \mat{W}\,\ten + \mu \vect{s} - \vect{w}_e = \vect{0},
    \end{cases}
\end{equation}
where $\vect{y} = [\ten^\top, \vect{s}^\top, \vect{\lambda}^\top]^\top$ optimization variables and time-dependence is made through the parameters $\mat{W}(t)$, $\vect{w}_e(t)$, and $\mu(t)$. 
%The map $\vect{r}(t, \vect{y})$ is $\mathcal{C}^{\min \{k_1, k_2, k_3 \} }$ from \ref{lemma:2}. 
Now, take its unique solution $(t^\star, \vect{y}^\star)$ and assume that the Jacobian $\mat{J}_{\vect{r}} = \partial \vect{r} / \partial \vect{y} \rvert_{t^\star, \vect{y}^\star}$ is full rank. Under Assumption~\ref{A:1} and Lemma~\ref{lemma:2}, the mapping $\vect{r}(t, \vect{y})$ is $\min\{k_1, k_2, k_3\}$-times differentiable in $(t, \vect{y})$. Then, according to the Implicit Function Theorem (IFT), there exists a neighborhood $U$ of $t^\star$ and a unique function $\xi$, $\min\{k_1, k_2, k_3\}$-times differentiable, such that $\xi(t^\star)=\vect{y}^\star$ and $\vect{r}(t, \xi(t))=\vect{0}$ for all $t \in U$.

The proof conclude showing that $\mat{J}_{\vect{r}}$ is full-rank for both values $\mu > 0$ and $\mu = 0$. When $\mu > 0$, one has
\begin{equation}\label{eq:block1}
\mat{J}_{\vect{r}}=
    \begin{bmatrix}
        \nabla_{\ten}^2 f & \vect{0} & \mat{W}^\top \\
        \vect{0} & 2 (\mu^2 \gamma + \eta)  \mathbb{1}_n & \mu \mathbb{1}_n \\
        \mat{W} & \mu \mathbb{1}_n & \vect{0}
    \end{bmatrix} = \begin{bmatrix}
        \mat{A}' & \mat{B}' \\
        \mat{B}'^\top & \mat{0}
    \end{bmatrix},
\end{equation}
with $\mat{A}'$ invertible and positive definite, as $\nabla_{\ten}^2 f$ and $2 (\mu^2 \gamma + \eta) \mathbb{1}_n$ both are, while $\mat{B}'$ is full-rank by \ref{A:1}. The nonsingularity can be concluded by Schur complement analysis.

In the case where $\mu = 0$, considering the rearranged version of the KKT matrix
\begin{equation}\label{eq:block2}
    \mat{J}_{\vect{r}'} = \begin{bmatrix}
        2 \eta  \mathbb{1}_n  & \vect{0} & \vect{0} \\
        \vect{0} & \nabla_{\ten}^2 f & \mat{W}^\top \\
        \vect{0} & \mat{W} & \vect{0}
    \end{bmatrix} = \begin{bmatrix}
        \mat{A}'' & \vect{0} \\
        \vect{0} & \mat{B}''
    \end{bmatrix},
\end{equation}
is clear that is nonsingular as its diagonal contains $\mat{A}''$ nonsingular for $\eta>0$ while $\mat{B}''$ is the KKT matrix of the AC method.
\end{proof}

\begin{remark}
    The differentiability of the tension profiles for the EAC is then limited by the functions introduced to extend the original AC. Indeed, when $\mu=0$, the differentiability depends on the smoothness of the input trajectory only. Notice, as for the AC, the inclusion of non-linear constraints, provided that they are convex, would automatically introduce a $k_4$ condition to consider for the smoothness of the tension profiles.
\end{remark}

\section{Simulations}\label{sec:simul}
To demonstrate the performance of the proposed approach, numerical experiments are conducted and compared to the other state-of-the-art methods. Firstly, a Monte Carlo analysis on the tuning parameters is reported to show their influence (Section~\ref{ssec:montecarlo}) when using the TDAs. After that, the performance of the method for a feasible (Section~\ref{ssec:feasible}) and an unfeasible (Section~\ref{ssec:unfeasible}) trajectory is studied. Finally, the ability to include nonlinear constraints is shown in Section~\ref{ssec:nl}.

Unless otherwise stated, the robot and trajectory are the same as those in~\cite{Ueland2020}. The corresponding trajectory parameters are reported in Tab.~\ref{tab:trajectory_definition}, while the robot parameters are reported in Tab.~\ref{tab:cdpr_all}.

The code used to generate the results is available in the repository associated with this article. 
Timing results were obtained by executing C++ compiled code from Python on a laptop equipped with an AMD R9 8945HS CPU and 32~GB of RAM, running Ubuntu 24.04.

For both methods, the Newton method with unfeasible start and Armijo backtracking is employed, as described in \cite{Boyd2004convex}. In addition to standard backtracking, a safety check is added during the damped phase to ensure that the solutions remain within $\Pi$.

\begin{table}[t]
\caption{CDPR system parameters and geometric configuration, as in \cite{Ueland2020}.}
\label{tab:cdpr_all}
\centering

%--------------------%
% Scalar parameters  %
%--------------------%
\begin{tabular}{llll}
\toprule
Parameter & Symbol & Value & Unit \\
\midrule
Degrees of freedom & $\mathrm{DoF}$ & 6 & -- \\
Number of cables & $m$ & 8 & -- \\
Inertias & $I_{xx}=I_{yy}=I_{zz}$ & 0.1 & $\mathrm{kg\,m^2}$ \\
Payload mass & $m_L$ & 1.0 & $\mathrm{kg}$ \\
Maximum cable tension & $\ub\ten$ & 40.0 & $\mathrm{N}$ \\
Minimum cable tension & $\lb\ten$ & 5.0 & $\mathrm{N}$ \\
\bottomrule
\end{tabular}

\vspace{0.5em}

%--------------------%
% Geometric points   %
%--------------------%
\setlength{\tabcolsep}{5pt}
\begin{tabular}{c|ccc|ccc}
\toprule
Cable $i$ &
\multicolumn{3}{c|}{$\vect{a}_i=[x,\ y,\ z]^T$ [m]} &
\multicolumn{3}{c}{$\vect{b}'_i=[x,\ y,\ z]^T$ [m]} \\
\midrule
1 & $-0.415$ & $-0.315$ & $-0.500$ & $-0.0525$ & $-0.0760$ & $0.0$ \\
2 & $-0.415$ & $-0.315$ & $ \phantom{-}0.500$ & $-0.0525$ & $-0.0760$ & $0.0$ \\
3 & $ \phantom{-}0.415$ & $-0.315$ & $ \phantom{-}0.500$ & $ \phantom{-}0.0525$ & $-0.0760$ & $0.0$ \\
4 & $ \phantom{-}0.415$ & $-0.315$ & $-0.500$ & $ \phantom{-}0.0525$ & $-0.0760$ & $0.0$ \\
5 & $ \phantom{-}0.415$ & $ \phantom{-}0.315$ & $-0.500$ & $0.0$ & $ \phantom{-}0.1240$ & $0.0$ \\
6 & $-0.415$ & $ \phantom{-}0.315$ & $ \phantom{-}0.500$ & $0.0$ & $ \phantom{-}0.1240$ & $0.0$ \\
7 & $-0.415$ & $ \phantom{-}0.315$ & $-0.500$ & $0.0$ & $ \phantom{-}0.1240$ & $0.0$ \\
8 & $ \phantom{-}0.415$ & $ \phantom{-}0.315$ & $ \phantom{-}0.500$ & $0.0$ & $ \phantom{-}0.1240$ & $0.0$ \\
\bottomrule
\end{tabular}

\end{table}

\begin{table}[t]
\caption{Reference trajectory parameters. The angles $\phi$, $\theta$ and $\psi$ are rotation angles using Euler convention. The position of the end-effector is defined as $\vect{p}=[x,\ y,\ z]$ while its orientation with $\vect{\psi}=[\phi,\ \theta,\ \psi]$. Trajectory is the same as in \cite{Ueland2020}.}
\label{tab:trajectory_definition}
\centering
\setlength{\tabcolsep}{6pt}
\begin{tabular}{ll}
\toprule
Item & Definition \\
\midrule
Total time & $T=10$~s \\
Physical time & $t \in [0,T]$ \\
Scaled path parameter & $s \in [0,1]$, quintic polynomial law \\
\midrule
$x(s)$ & $0.1\sin(\pi s)$ \\
$y(s)$ & $0.1\cos(\pi s)$ \\
$z(s)$ & $0.34\sin(\pi s)\cos(2\pi s)-0.052$ \\
$\phi(s)$ & $-0.2\,s\cos(\pi s)$ \\
$\theta(s)$ & $0$ \\
$\psi(s)$ & $-0.17\sin(\pi s)\cos(2\pi s)$ \\
\bottomrule
\end{tabular}
\end{table}

% \begin{figure}[t!]
%     \centering
%     \includegraphics[width=\linewidth]{Figures/basetraj.eps}
%     \caption{Base trajectory used in the experiments. The angles $\phi$, $\theta$ and $\psi$ are rotation angles using Euler convention. The position of the end-effector is defined as $\vect{p}=[x,\ y,\ z]$ while its orientation with $\vect{\psi}=[\phi,\ \theta,\ \psi]$.}\label{fig:basetraj}
% \end{figure}

\subsection{Effect of tuning parameters}\label{ssec:montecarlo}
As anticipated, the EAC requires tuning $\delta$, $\gamma$, $\eta$ parameters and for them we suggest using $\delta=0.1$~N, $\gamma=5\cdot 10^2$, and $\eta=10^{-5}$. However, the sensitivity analysis is needed to prove this choice is possibly an optimal trade-off over a wide range of values. At the same time, we do compare the sensitivity of the EAC against the NTNU method. To reduce the search space, $c_{1,i}=c_{2,i}=0.1$ are fixed as suggested in~\cite{Ueland2020}, leaving only $\epsilon$ and $b$ to be tuned.

From a wrench-error standpoint, the two methods reach similar performances, with minimum values that are practically identical. In Fig.~\ref{fig:heatmapwrench}a-b, the maximum error during the trajectory is displayed for both methods while varying the tuning parameters. It is clear that both methods have a large domain in which the error is minimal, thus not imposing strict constraints on their values. 

\begin{figure*}[t!]
    \centering
    \includegraphics[width=\linewidth]{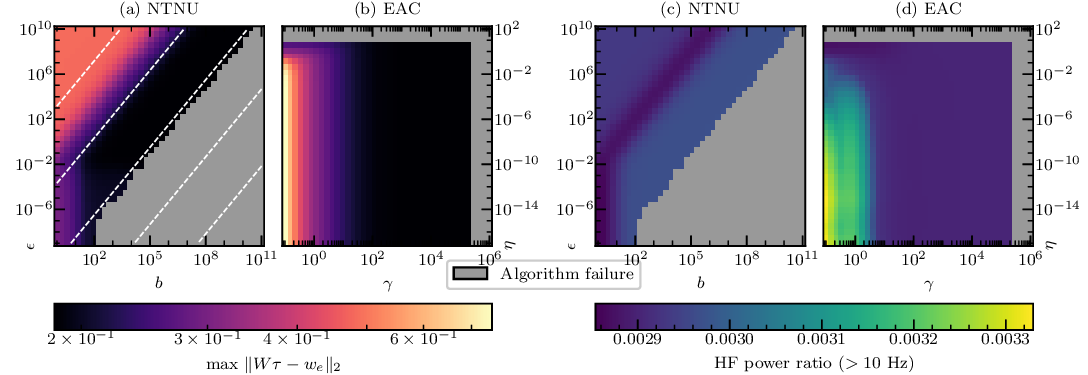}
    \caption{Max 2-norm error of converged solution for (a) NTNU method and (b) the proposed method. In the NTNU plot, white dashed lines represent the equation $\sqrt{\epsilon}/b=\mathrm{const}$. Conversely, (c) and (d) depict the fraction of frequency content above $10$~Hz for the NTNU and proposed methods, respectively. Not converged solutions are displayed in grey.}\label{fig:heatmapwrench}
\end{figure*}

Another metric to evaluate the TDAs performances, in addition to wrench error, is the smoothness of the tension profiles. Smoothness is necessary to avoid exciting high-frequency components of the cables, leading to possible vibrations of the load. For this reason, Fig.~\ref{fig:heatmapwrench}c-d reports the ratio of the power spectral density (PSD) above a certain threshold, here chosen as $f_t = 10$~Hz for both methods. 

By comparing Fig.~\ref{fig:heatmapwrench}a-b with Fig.~\ref{fig:heatmapwrench}c-d, it is clear that the region of minimal wrench error for the NTNU method corresponds to higher frequency content, while the EAC method shows lower frequency content in the area with minimal wrench error. This phenomenon is further confirmed by observing the Pareto front for the two methods in terms of error and smoothness, provided in Fig.~\ref{fig:paretofreq}.

\begin{figure}[t!]
    \centering
    \includegraphics[width=\linewidth]{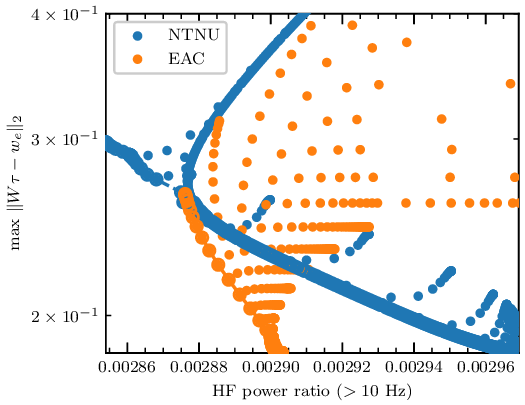}
    \caption{Pareto fronts for the two methods, PSD content above $10$~Hz vs $2-$ norm wrench error.}\label{fig:paretofreq}
\end{figure}

Lastly, the tuning parameters affect the convergence speed of both methods, as they affect the conditioning of the KKT matrix. This effect is studied in Fig.~\ref{fig:timing}a-b, where the maximum time to converge is reported for both methods while varying their tuning parameters; notably, the worst time to converge is comparable for the two methods, with the EAC slightly outperforming the NTNU one. Apart from the maximum time to converge, used for real-time operation, other control methods require computing the reference tension profiles in advance. In this setting, the mean time is typically used~\cite{liang2026}, and thus the convergence time is depicted in Fig.~\ref{fig:timing}c-d for both methods. It is evident that the EAC achieves faster convergence.

\begin{figure*}[t!]
    \centering
    \includegraphics[width=\linewidth]{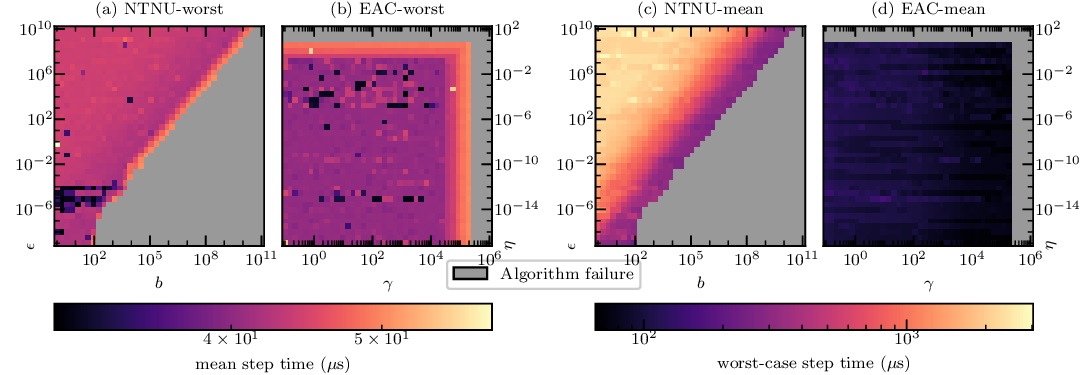}
    \caption{Worst time to converge for (a) the NTNU method and (b) for the proposed method, alongside the mean times for (c) the NTNU method and (d) for the proposed method.}\label{fig:timing}
\end{figure*}

\subsection{Example inside WFW}\label{ssec:feasible}

In the following, we fix the tuning parameters to $\gamma=10^{4}$ and $\eta=10^{-5}$ for the proposed method, while choosing $b=10^{6}$ and $\epsilon=10^{3}$ for the NTNU method. These values are selected to minimize the wrench error while being far from the non-convergence area. 

Instead of the base trajectory used in the previous study, we now consider a circular path, with a radius $r=0.1$~m, centered at $C=[0,\; 0,\; -0.05]^\top$~m, with a total execution time $T=1.4$~s, planned using a quintic polynomial. 

A comparison of the two methods is reported in Fig.~\ref{fig:feasible}. Notably, Fig.~\ref{fig:feasible}a-b show that both trajectories approach the tension limits smoothly. Interestingly, Fig.~\ref{fig:feasible}c shows that, unlike the NTNU method, the proposed method does not introduce errors in the wrench apart from numerical noise.
Lastly, Fig.~\ref{fig:feasible}d shows the $99^\textrm{th}$ percentile of the execution time over $5000$ runs. Such a metric has been chosen to reduce the impact of non-deterministic OS scheduling effects while preserving representative worst-case behavior. Given that the control law is deterministic and evaluated at identical states, the observed variability is ascribed to exogenous system-level factors rather than to the algorithm itself.

\begin{figure*}[t]
    \centering
    \includegraphics[width=\linewidth]{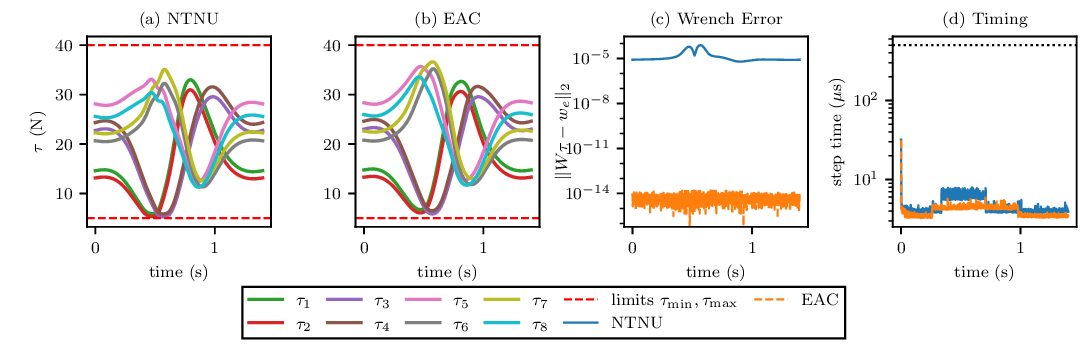}
    \caption{Feasible trajectory: (a) depicts the tension profiles for the NTNU method while (b) the ones for the proposed method. In (c) wrench errors are compared alongside the $99^\textrm{th}$ percentile of time to converge in (d). For convenience, the $500~\mu$s timing level is highlighted with a black dotted line.}\label{fig:feasible}
\end{figure*}

The worst time to converge, in the $99^\textrm{th}$ percentile sense, was $20.46~\mu$s for the NTNU method, while the proposed method required $18.71~\mu$s, thus resulting in a $9\%$ reduction.

\subsection{Example outside WFW}\label{ssec:unfeasible}
An additional test has been conducted to assess the performance of the proposed method, by pushing it outside the WFW. The same setting of Section~\ref{ssec:feasible} is adopted, with the only modification being an increased radius of the trajectory, namely $r=0.12$~m. The increased accelerations and the modified configuration lead to unfeasible conditions. The comparison of the two methods is reported in Fig.~\ref{fig:unfeasible}. Notably, both methods approach the tension limits smoothly, as shown in Fig.~\ref{fig:unfeasible}a-b. The wrench error is depicted in Fig.~\ref{fig:unfeasible}c, where it can be noted that for unfeasible points it is comparable, while for feasible points the EAC outperforms the NTNU. Moreover, the computational time is significantly lower and never exceeds $500~\mu$s. Accordingly, the worst time to converge, in the $99^\textrm{th}$ percentile sense, was $511~\mu$s for the NTNU method, while the proposed method required $244~\mu$s, thus resulting in a $52\%$ reduction.

\begin{figure*}[t]
    \centering
    \includegraphics[width=\linewidth]{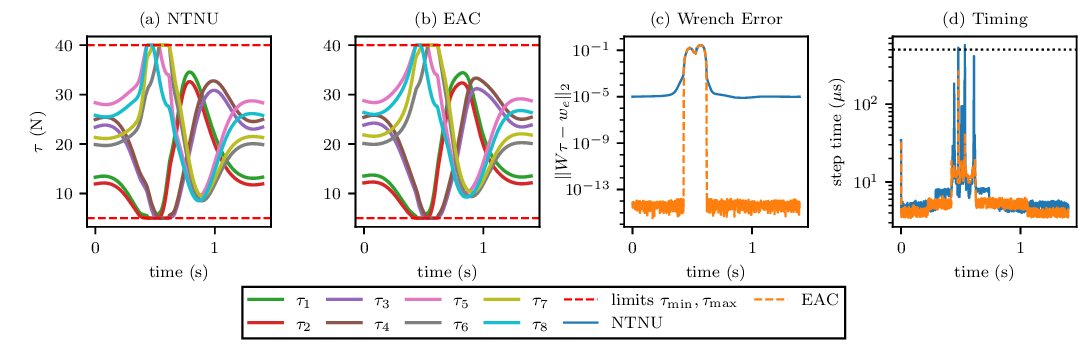}
    \caption{Unfeasible trajectory: (a) depicts the tension profiles for the NTNU method while (b) the ones for the proposed method. In (c) wrench errors are compared alongside the $99^\textrm{th}$ percentile of time to converge in (d). For convenience, the $500~\mu$s timing level is highlighted with a black dotted line.} \label{fig:unfeasible}
\end{figure*}

\begin{figure*}[t]
    \centering
    \includegraphics[width=\linewidth]{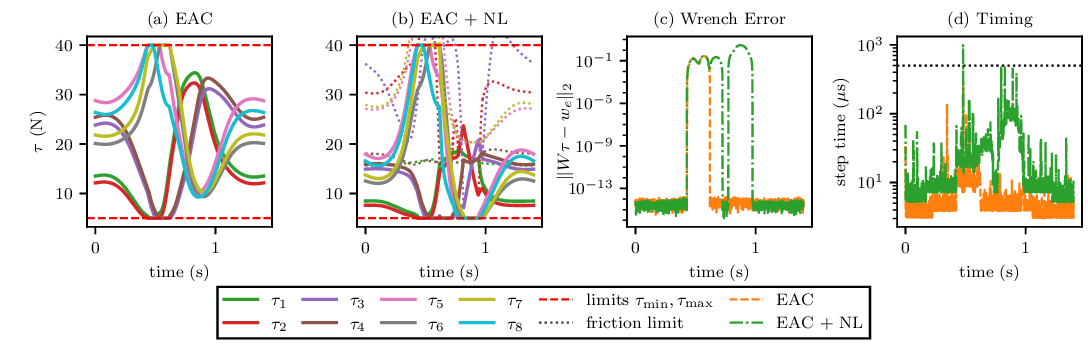}
    \caption{Unfeasible trajectory+NL constraints: (a) depicts the tension profiles without NL constraints (b) the ones with the the NL constraints alongside the NL limits in dotted line. In (c) wrench errors are compared alongside the $99^\textrm{th}$ percentile of time to converge in (d). For convenience, the $500~\mu$s timing level is highlighted with a black dotted line.} \label{fig:unfeasibleNL}
\end{figure*}

\subsection{Inclusion of non-linear constraints}\label{ssec:nl}
One of the advantages of the log-barrier formulation is the inclusion of non-linear inequality constraints. This feature is not available in ICF-based or QP-based methods.
Therefore, a test is conducted to show the EAC capabilities, assuming the cables are attached to four different structures (see, for example, Fig.~\ref{fig:cdprscheme}), which are secured by means of static friction.
Specifically, in the $k$-th structure, the equilibrium in the $z$-direction is
\begin{equation}
    n_k-m_s g - \sum_i^{\mathcal{I}_k} u_{z,i} \tau_i = 0,
\end{equation}
where $n_k$ is the ground reaction force, $m_s$ is the mass of the structure (in this example $m_s=4$~kg), and $\mathcal{I}_k$ are the indices of the cables attached for the $k$-th structure. The equilibrium in the $x$ and $y$ directions reads
\begin{equation}
    \begin{gathered}
        f_{x,k} - \sum_i^{\mathcal{I}_k} u_{x,i} \tau_i = 0 \quad \text{and} \quad f_{y,k} - \sum_i^{\mathcal{I}_k} u_{y,i} \tau_i = 0.
    \end{gathered}    
\end{equation}
Thus, the no-slip condition reads as
\begin{equation}\label{eq:NLcons}
    h_k(\ten) = f_c^2 n_k^2 - f_{x,k}^2 - f_{y,k}^2 > 0.
\end{equation}
where $f_c$ is the dry friction coefficient, in this work set equal to $0.5$.
The conditions~\eqref{eq:NLcons} are included in the objective of both~\eqref{eq:def} and\eqref{eq:def2} either as $\beta_\delta(h_k(\ten))$ for the former or $-\log(h_k(\ten))$.

The results reported in Fig.~\ref{fig:unfeasibleNL} show the influence of the friction constraints on cable tensions. In particular, their presence leads to lower tension values along the trajectory, thus avoiding critical crashes of the CDPR. Notice that in the middle of the simulations, the slack variable is active. This provides evidence of the EAC' abilities to manage both non-linear constraints while going outside the WFW.

\section{Conclusion}\label{sec:conclusion}
In this paper, the problem of generating smooth tension profiles for CDPRs, while guaranteeing zero wrench error inside the WFW has been addressed.

First, the limitations of existing approaches have been discussed. In particular, QP-based methods inherently introduce non-differentiable points and are not naturally suited to handle nonlinear constraints, while heuristic approaches generally do not guarantee continuity and may lead to significant wrench errors. Moreover, the existing NTNU barrier-function-based formulations have been shown to exhibit either reduced accuracy inside the WFW or slower convergence.

To overcome these limitations, a formulation capable of combining these features has been proposed. The method is based on a two-step approach: a relaxed barrier-function formulation is first employed as the nominal solution, while a slack-based formulation is activated whenever the solution exits a reduced tension box.

Future work will focus on the experimental validation of the proposed approach in real-world control scenarios.

%%%%%%%%%%%%%  BIBLIOGRAPHY  %%%%%%%%%%%%%%%%%%%%%%%%%%%%%%%%%%%%%%%%%
\nocite{*}
\bibliographystyle{ieeetr}   %% .bst file that follows ASME journal format. Do not change.
\bibliography{bibliography}

\end{document}